\documentclass[11pt]{amsart}

\usepackage[T1]{fontenc}
\usepackage[utf8]{inputenc}
\usepackage{textcomp}
\usepackage{lmodern}
\usepackage{microtype}

\usepackage{amssymb}
\usepackage{amsthm}
\usepackage{amscd}

\usepackage{mathscinet}
\usepackage{hyperref}
\usepackage{url}

\usepackage{tikz}

\numberwithin{equation}{section}

\makeatletter
\newtheorem*{rep@theorem}{\rep@title}
\newcommand{\newreptheorem}[2]{%
\newenvironment{rep#1}[1]{%
 \def\rep@title{#2 \ref{##1}}%
 \begin{rep@theorem}}%
 {\end{rep@theorem}}}
\makeatother
\newtheorem{theorem}{Theorem}[section]

\newreptheorem{prop}{Proposition}

\newtheorem{cor}[theorem]{Corollary}
\newtheorem{fac}[theorem]{Fact}

\theoremstyle{definition}
\newtheorem{definition}[theorem]{Definition}

\newtheorem{remark}[theorem]{Remark}

\newtheorem*{claim*}{Claim}

  \def \B{\operatorname{\mathcal{B}}}

  \def \VC{\operatorname{VC}}

      \def \PAC{\operatorname{PAC}}

    \def \fin{\operatorname{dist}}

 \def \ttimes{\times \ldots \times}

\allowdisplaybreaks[2]

\begin{document}

\title{Higher-arity PAC learning, VC dimension and packing lemma}
\author{Artem Chernikov and Henry Towsner}
\date{\today}

\begin{abstract}
  The aim of this note is to overview some of our work in \cite{chernikov2020hypergraph}  developing higher arity VC theory (VC$_n$ dimension), including a generalization of Haussler packing lemma, and an associated tame (slice-wise) hypergraph regularity lemma; and to demonstrate that it characterizes higher arity PAC learning (PAC$_n$ learning) in $n$-fold product spaces with respect to product measures  introduced by Kobayashi, Kuriyama and Takeuchi \cite{Kobayashi, KotaPAC}. We also point out how some of the recent results in \cite{coregliano2024high, coregliano2025packing, coregliano2025tt} follow from our work in \cite{chernikov2020hypergraph}.
\end{abstract}

\maketitle
%

\section{Introduction}

In \cite{chernikov2020hypergraph} (following some preliminary work in \cite{chernikov2024nPreprint}), we  developed the foundations of higher arity VC-theory, or VC$_k$-theory, for families of subsets in $k$-fold product spaces (and in fact more generally, for families of real valued functions). The notion of  VC$_k$ dimension (Definition \ref{def: VCk dim}) is implicit in Shelah's work on \emph{$k$-dependent theories} in model theory \cite{MR3666349, MR3273451}, and is studied quantitatively in \cite{chernikov2024nPreprint} (published as \cite{chernikov2014n}) and further on the model theory side in \cite{MR3509704, chernikov2017mekler, chernikov2019n, chernikov2024n}. In particular, answering a question of Shelah from \cite{MR3273451}, \cite{chernikov2024nPreprint} established an appropriate generalization of Sauer-Shelah lemma for $\VC_k$ dimension (Fact \ref{fac: n-dep Sauer-Shelah}; see Remark \ref{rem: hist of Saur-Shelah} for its history). Following this, \cite{Kobayashi, KotaPAC} proposed a higher arity generalization of $\PAC$ learning to $\PAC_k$ learning for families of sets in $k$-fold product spaces (Section \ref{sec: PACn learning}).

One of the main result in \cite{chernikov2020hypergraph} is a higher arity generalization of Haussler's packing lemma from $\VC$ to $\VC_k$-dimension (Fact \ref{fac: finite VCk-dim implies approx bounded}); and its converse, that packing lemma uniformly over all measures implies finite $\VC_k$-dimension, see Remark \ref{rem: VC_k dim and packing lemma equiv}). It was used in \cite{chernikov2020hypergraph} to obtain a tame regularity lemma for hypergraphs of finite $\VC_k$-dimension (generalizing from $\VC$-dimension and graphs \cite{alon2007efficient, lovasz2010regularity, hrushovski2013generically}; an analogous hypergraph regularity lemma for $k=3$ was established in \cite{terry2021irregular} using different methods).

Our main new observation here is that the packing lemma from \cite{chernikov2020hypergraph}  quickly implies the equivalence of finite VC$_k$-dimension and PAC$_k$-learnability:
\begin{cor}[Corollary \ref{cor: everything equiv}]
	The following are equivalent for a class $\mathcal{F}$ of subsets of $V_1 \times \ldots \times V_k$:
	\begin{enumerate}
		\item $\mathcal{F}$ has finite $\VC_k$-dimension (Definition \ref{def: VCk dim});
		\item $\mathcal{F}$ satisfies the packing lemma (in the sense of Fact \ref{fac: finite VCk-dim implies approx bounded});
		\item $\mathcal{F}$ is (properly) $\PAC_k$-learnable (in the sense of Definition \ref{def: PAC_k learn}).
	\end{enumerate}
\end{cor}

 We prove it in Section \ref{sec: VCk PACk} from our packing lemma for $\VC_k$-dimension analogously to the proof that Haussler's packing lemma implies PAC learnability. This  observation was presented at the Harvard Center of Mathematical Sciences and Applications colloquium in October 2024 and  Carnegie Mellon University Logic Seminar in November 2024.


There appears to be some recent interest in higher arity VC theory \cite{coregliano2024high, coregliano2025packing, coregliano2025tt}, which includes reproving some results equivalent to those in \cite{chernikov2024nPreprint}, \cite{KotaPAC} and \cite{chernikov2020hypergraph}. In particular, some of the main results in these papers follow from or are implicit in \cite{chernikov2024nPreprint}, \cite{KotaPAC} and \cite{chernikov2020hypergraph}, at least qualitatively (the latter paper did not investigate the bounds). In Section \ref{sec: slicewise pack} of this paper we state a slice-wise version of our packing lemma for VC$_k$ dimension (Corollary \ref{cor: slice-wise packing lemma}, implicit in our proof that packing lemma implies slice-wise hypergraph regularity for $\VC_k$ in \cite{chernikov2020hypergraph}) and note its equivalence in the case $k = 1$ to the main result of \cite{coregliano2025packing} (which is a packing lemma for relations of finite slice-wise VC-dimension, relying on the main results of \cite{coregliano2024high}). While this note was in preparation, \cite{coregliano2025tt} considered a version of PAC$_k$-learnability equivalent to the one in \cite{Kobayashi, KotaPAC}, and announced a result analogous to Corollary \ref{cor: everything equiv} (again, relying on a packing lemma equivalent to ours from \cite{chernikov2020hypergraph}). See also Remark \ref{rem: hist of Saur-Shelah}.

\subsection*{Acknowledgements}
Chernikov was partially supported by the NSF Research Grant DMS-2246598; and by the Deutsche Forschungsgemeinschaft (DFG, German Research Foundation) under Germany's Excellence Strategy -- EXC-2047/1 -- 390685813. He thanks Hausdorff Research Institute for Mathematics in Bonn, which hosted him during the Trimester Program ``Definability, decidability, and computability'' where part of this paper was written.  Towsner was supported by NSF grant DMS-2452009.

\section{VC$_k$ dimension}
We review the notion of VC$_k$-dimension for families of subsets of $k$-fold product sets $V_1 \times \ldots \times V_k$, for $k \in \mathbb{N}$, generalizing the usual Vapnik-Chervonenkis dimension in the case $k=1$. It is implicit in Shelah's work on \emph{$k$-dependent theories} in model theory \cite{MR3666349, MR3273451}, and is formally introduced and studied quantitatively in \cite{chernikov2024nPreprint}.  For $n \in \mathbb{N}$, we write $[n] := \{1, \ldots, n\}$.

\begin{definition}\label{def: VCk dim}
For $k \in \mathbb{N}$, let $V_1, \ldots, V_{k}$ be sets. Let $\mathcal{F} \subseteq V_1 \times \ldots V_k$ be a family of sets.  We say that $\mathcal{F}$ has \emph{$\VC_k$-dimension $\geq d$}, or $\VC_k(E) \geq d$, if there is a \emph{$k$-dimensional $d$-box} $A = A_1 \ttimes A_k$ with $A_i \subseteq V_i$ and $|A_i|=d$ for $i=1, \ldots, k$ \emph{shattered} by $\mathcal{F}$. That is, for every $A' \subseteq A$, there is some $S \in  \mathcal{F}$ such that $A' = A \cap S$.
We say that $\VC_k(\mathcal{F}) = d$ if $d$ is maximal such that there is a $d$-box shattered by $\mathcal{F}$, and $\VC_k(\mathcal{F}) = \infty$ if there are $d$-boxes shattered by $\mathcal{F}$ for arbitrarily large $d$. 

Given sets $V_1, \ldots, V_{k+1}$ and $E \subseteq V_1 \times \ldots \times V_{k+1}$, we let $\VC_k(E) := \VC_k(\mathcal{F}_E)$, where $\mathcal{F}_E = \{E_b \subseteq V_1\times \ldots \times V_{k} : b \in V_{k+1}\}$ and $E_b = \{(b_1, \ldots, b_k) \in V_1 \times \ldots \times V_k : (b_1, \ldots, b_k, b) \in E\}$ is the slice of $E$ at $b$.
\end{definition}

In the case $k=1$ and $\mathcal{F} \subseteq V_1 $, $\VC_1(\mathcal{F}) = d$ simply means that the family $\mathcal{F} $ has $\VC$-dimension $d$. 

In \cite{chernikov2020hypergraph}, we also consider a generalization of $\VC_k$-dimension for families of real valued functions rather than just sets.

\section{Sauer-Shelah lemma for VC$_k$-dimension}

The following is a generalization of the Sauer-Shelah lemma from VC$_1$ to VC$_k$-dimension:

\begin{fac}\label{fac: n-dep Sauer-Shelah}\cite[Proposition 3.9]{chernikov2024nPreprint} 
If $\mathcal{F} \subseteq V_1 \ttimes V_{k+1}$ satisfies $\VC_{k}(\mathcal{F}) < d$, then there is some $\varepsilon = \varepsilon(d) \in \mathbb{R}_{>0}$ such that: for any $A = A_1 \ttimes A_{k} \subseteq V_1 \ttimes V_k$ with $|A_1| = \ldots = |A_{k}| = m$, there are at most $2^{m^{k - \varepsilon}}$ different sets $A' \subseteq A$ such that $A' = A \cap S$ for some $S \in \mathcal{F}$.
\end{fac}

\begin{remark}
More precisely, if $\VC_k \leq d$, then the upper bound above in \cite[Proposition 3.9]{chernikov2024nPreprint} is actually given by $\sum_{i<z} {m^k \choose i} \leq 2^{m^{k-\varepsilon}}$ for $m \geq k$, where $z = z_k(m, d+1)$ is the Zarankiewicz number, i.e.~the minimal natural number $z$ satisfying: every $k$-partite $k$-hypergraph with parts of size $m$ and $\geq z$ edges contains the complete $k$-partite hypergraph with each part of size $d+1$. If $k=1$, then $z_1(m,d+1) = d+1$, hence the bound in Fact \ref{fac: n-dep Sauer-Shelah}  coincides with the Sauer-Shelah bound, and for a general $k$ the exponent in Fact \ref{fac: n-dep Sauer-Shelah} appears close to optimal (see \cite[Proposition 3.9]{chernikov2024nPreprint} for the details).
\end{remark}

\begin{remark}\label{rem: hist of Saur-Shelah}
	In \cite[Conclusion 5.66]{MR3273451} Shelah observed the converse implication: if for infinitely many $m$, there are at most $2^{m^{k - \varepsilon}}$ different sets $A' \subseteq A$ such that $A' = A \cap S$ for some $S \in \mathcal{F}$, then the family has finite VC$_k$ dimension; and asked \cite[Question 5.67(1)]{MR3273451} if Fact \ref{fac: n-dep Sauer-Shelah} was true.  Hence \cite[Proposition 3.9]{chernikov2024nPreprint} answered Shelah's question, also demonstrating its qualitative optimality. The discussion in \cite[Page 10 and Section 7]{coregliano2025tt} misstates this and fails to acknowledge \cite{chernikov2024nPreprint}.
\end{remark}

\section{PAC$_n$ learning}\label{sec: PACn learning}
For simplicity of presentation we are going to ignore  measurability issues here and just restrict to arbitrarily large finite probability spaces. However, all of the results from \cite{chernikov2020hypergraph} cited in this note hold at the level of generality of (partite) \emph{graded probability spaces} (which includes both countable/separable families in products of Borel spaces and arbitrarily large finite probability spaces/their ultraproducts), we refer to \cite[Section 2.2]{chernikov2020hypergraph} for a detailed discussion of the setting.

We first recall the classical PAC learning:
\begin{definition}
Let $\mathcal{F}$ be a family of measurable subsets of a probability space $(V, \mathcal{B},\mu)$, which we also think about as their indicator functions $V \to \{0,1\}$.

	\begin{enumerate}
	\item For $\bar{a} \in V^m$, let $f \restriction_{\bar{a}} := ((a_1, f(a_1)), \ldots, (a_m, f(a_m)))$.
		\item Let $\mathcal{F}_{\fin} := \left\{  f \restriction_{\bar{a}} : f \in \mathcal{F}, \bar{a} \in V^m, m \in \mathbb{N} \right\}$.
	\item A function $H : \mathcal{F}_{\fin} \to \mathcal{B}$ is a \emph{learning function} for $\mathcal{F}$ if for every $\varepsilon,\delta >0$ there is $N_{\varepsilon, \delta} \in \mathbb{N}$ satisfying:
	for every $m \geq N_{\varepsilon, \delta}$, probability measure $\mu$ on $(V,\mathcal{B})$ and $f \in \mathcal{F}$ we have
	$$\mu^{\otimes m}\left( \left\{ \bar{a} \in V^m : \mu \left( H \left(f \restriction_{\bar{a}}  \right) \Delta f \right) > \varepsilon \right\} \right) \leq \delta.$$
	\item 	A learning function $H$ is \emph{proper} if its image is contained in $\mathcal{F}$.
	\item 	$\mathcal{F}$ is (properly) \emph{PAC-learnable} if $\mathcal{F}$ admits a (proper) learning function $H$.
	
\end{enumerate}
\end{definition}

\begin{fac}\cite{blumer1989learnability}
 A class $\mathcal{F}$ is (properly) PAC-learnable if and only if $\VC(\mathcal{F}) < \infty$.
\end{fac}

Motivated by the work in \cite{chernikov2024nPreprint} on $\VC_n$-dimension, Kobayashi, Kuriyama and Takeuchi \cite{Kobayashi, KotaPAC} proposed a higher arity generalization of $\PAC$ learning:
\begin{definition}\label{def: PAC_k learn}
	\begin{enumerate}

\item Let $n \in \mathbb{N}_{\geq 1}$ be fixed, $V = V_1 \ttimes V_{n}$ and $\mathcal{F}$ a family of subsets of $V$.
\item For a tuple $\bar{a} = (a_1, \ldots, a_n) \in V$, let $D_n(\bar{a})$ be the set 
$$\bigcup_{i \in [n]}\{(x_1, \ldots, x_{i-1}, a_{i}, x_{i+1}, \ldots, x_n) \in V : x_j \in V_j\}.$$
\item For a set $S \subseteq V$, let $f \restriction_{S} := \{ (\bar{x}, f(\bar{x})) : \bar{x} \in S \}$. And for a
	\item For any $m \in \mathbb{N}$ and tuple $\vec{a} = (\bar{a}_1, \ldots, \bar{a}_{m}) \in V^m$, where $\bar{a}_i = (a_{i,1}, \ldots, a_{i,n}) \in V$, and $f \in \mathcal{F}$, we let 
	$$f \restriction_{D_n(\vec{a})} := \Big(  f \restriction_{D_n(\bar{a}_1)}, \ldots,   f \restriction_{D_n(\bar{a}_m)}  \Big).$$

		\item Let 
		$$\mathcal{F}^n_{\fin} := \left\{  f \restriction_{D_n(\vec{a})}: f \in \mathcal{F}, \vec{a} \in V^m, m \in \mathbb{N} \right\}.$$
	\item A function $H : \mathcal{F}^n_{\fin} \to \mathcal{F}$ is a \emph{learning function} for $\mathcal{F}$ if for every $\varepsilon,\delta >0$ there is $N_{\varepsilon, \delta} \in \mathbb{N}$ satisfying:
	for every $m \geq N_{\varepsilon, \delta}$, probability measures $\mu_i$ on $(V_i,\mathcal{B}_i)$ and $f \in \mathcal{F}$ we have
	$$\mu^{\otimes m}\left( \left\{ \vec{a} \in V^m : \mu \left( H \left(f \restriction_{D_n(\vec{a})} \right)  \Delta f \right) > \varepsilon \right\} \right) \leq \delta,$$
	where $\mu$ is the product measure $\mu_1 \otimes \ldots \otimes \mu_{n}$ on $V^m$.
	\item A learning function $H$ is \emph{proper} if its image is contained in $\mathcal{F}$.
	\item $\mathcal{F}$ is (properly) \emph{PAC$_n$-learnable} if $\mathcal{F}$ admits a (proper) learning function $H$.
\end{enumerate}
\end{definition}

\begin{remark}\label{rem: reading fibers from Takeuchi}

Assuming $|V_i|\geq 2$ for all $i$, given $D_n(\bar{a})$, we can recover $\bar{a}= (a_1, \ldots, a_m)$ (by asking for the common $i$th coordinate of any two points in $D(\bar{a})$ with all other coordinates pairwise distinct) and the sets $\left\{ (x_1, \ldots, x_{i-1}, a_i, x_{i+1}, \ldots, x_{n}) \in V : x_j \in V_j  \right\}$ for all $i$. It follows that given $f \restriction_{D_n(\bar{a})}$, we can recover all $\bar{a}$-slices of $f$ of arity $\leq n-1$.
\end{remark}

The following is an illustration for $\PAC_2$ learning. The family of hypothesis $\mathcal{F}$ is given. An adversary picks some measures $\mu_1$ on $V_1$ and $\mu_2$ on $V_2$. Some points $a_{1,1}, \ldots, a_{N,1}$ are sampled  from $V_1$ with respect to $\mu_1$, and $a_{1,2}, \ldots, a_{N,2}$ are sampled from $V_2$ with respect to $\mu_2$.
  Then an adversary picks a set $f \subseteq V_1 \times V_2$ from $\mathcal{F}$, and we are given the (bounded by $2N$ number of) $1$-dimensional slices of $f$, in either direction, with the fixed coordinate coming from one of the  points $(a_{i,1}, a_{i,2})$ in our sample.  The learning function aims to recover, given this information, the set $f$ up to small symmetric difference with respect to the product measure $\mu_1 \otimes \mu_2$:

\tikzset{every picture/.style={line width=0.75pt}} 

\begin{tikzpicture}[x=0.75pt,y=0.75pt,yscale=-1,xscale=1]

\draw   (23.36,30.64) -- (132.38,30.64) -- (132.38,139.66) -- (23.36,139.66) -- cycle ;
\draw    (110.26,18.79) -- (108.06,45.24) ;
\draw [shift={(107.89,47.23)}, rotate = 274.76] [color={rgb, 255:red, 0; green, 0; blue, 0 }  ][line width=0.75]    (10.93,-3.29) .. controls (6.95,-1.4) and (3.31,-0.3) .. (0,0) .. controls (3.31,0.3) and (6.95,1.4) .. (10.93,3.29)   ;
\draw   (163.98,30.64) -- (272.99,30.64) -- (272.99,139.66) -- (163.98,139.66) -- cycle ;
\draw    (250.87,18.79) -- (248.67,45.24) ;
\draw [shift={(248.5,47.23)}, rotate = 274.76] [color={rgb, 255:red, 0; green, 0; blue, 0 }  ][line width=0.75]    (10.93,-3.29) .. controls (6.95,-1.4) and (3.31,-0.3) .. (0,0) .. controls (3.31,0.3) and (6.95,1.4) .. (10.93,3.29)   ;
\draw  [color={rgb, 255:red, 208; green, 2; blue, 27 }  ,draw opacity=1 ][fill={rgb, 255:red, 208; green, 2; blue, 27 }  ,fill opacity=1 ] (176.79,140.01) .. controls (176.79,138.8) and (177.77,137.82) .. (178.98,137.82) .. controls (180.19,137.82) and (181.17,138.8) .. (181.17,140.01) .. controls (181.17,141.22) and (180.19,142.2) .. (178.98,142.2) .. controls (177.77,142.2) and (176.79,141.22) .. (176.79,140.01) -- cycle ;
\draw  [color={rgb, 255:red, 208; green, 2; blue, 27 }  ,draw opacity=1 ][fill={rgb, 255:red, 208; green, 2; blue, 27 }  ,fill opacity=1 ] (193.38,140.4) .. controls (193.38,139.19) and (194.36,138.21) .. (195.57,138.21) .. controls (196.78,138.21) and (197.76,139.19) .. (197.76,140.4) .. controls (197.76,141.61) and (196.78,142.59) .. (195.57,142.59) .. controls (194.36,142.59) and (193.38,141.61) .. (193.38,140.4) -- cycle ;
\draw  [color={rgb, 255:red, 208; green, 2; blue, 27 }  ,draw opacity=1 ][fill={rgb, 255:red, 208; green, 2; blue, 27 }  ,fill opacity=1 ] (240.38,139.61) .. controls (240.38,138.4) and (241.36,137.42) .. (242.57,137.42) .. controls (243.78,137.42) and (244.76,138.4) .. (244.76,139.61) .. controls (244.76,140.82) and (243.78,141.8) .. (242.57,141.8) .. controls (241.36,141.8) and (240.38,140.82) .. (240.38,139.61) -- cycle ;
\draw  [color={rgb, 255:red, 208; green, 2; blue, 27 }  ,draw opacity=1 ][fill={rgb, 255:red, 208; green, 2; blue, 27 }  ,fill opacity=1 ] (161.78,125) .. controls (161.78,123.79) and (162.76,122.81) .. (163.97,122.81) .. controls (165.18,122.81) and (166.16,123.79) .. (166.16,125) .. controls (166.16,126.21) and (165.18,127.19) .. (163.97,127.19) .. controls (162.76,127.19) and (161.78,126.21) .. (161.78,125) -- cycle ;
\draw  [color={rgb, 255:red, 208; green, 2; blue, 27 }  ,draw opacity=1 ][fill={rgb, 255:red, 208; green, 2; blue, 27 }  ,fill opacity=1 ] (161.78,108.01) .. controls (161.78,106.8) and (162.76,105.83) .. (163.97,105.83) .. controls (165.18,105.83) and (166.16,106.8) .. (166.16,108.01) .. controls (166.16,109.22) and (165.18,110.2) .. (163.97,110.2) .. controls (162.76,110.2) and (161.78,109.22) .. (161.78,108.01) -- cycle ;
\draw  [color={rgb, 255:red, 208; green, 2; blue, 27 }  ,draw opacity=1 ][fill={rgb, 255:red, 208; green, 2; blue, 27 }  ,fill opacity=1 ] (161.78,57.85) .. controls (161.78,56.64) and (162.76,55.66) .. (163.97,55.66) .. controls (165.18,55.66) and (166.16,56.64) .. (166.16,57.85) .. controls (166.16,59.06) and (165.18,60.04) .. (163.97,60.04) .. controls (162.76,60.04) and (161.78,59.06) .. (161.78,57.85) -- cycle ;
\draw   (304.86,31.17) -- (413.87,31.17) -- (413.87,140.19) -- (304.86,140.19) -- cycle ;
\draw  [color={rgb, 255:red, 208; green, 2; blue, 27 }  ,draw opacity=1 ][fill={rgb, 255:red, 208; green, 2; blue, 27 }  ,fill opacity=1 ] (317.67,140.53) .. controls (317.67,139.33) and (318.65,138.35) .. (319.86,138.35) .. controls (321.07,138.35) and (322.05,139.33) .. (322.05,140.53) .. controls (322.05,141.74) and (321.07,142.72) .. (319.86,142.72) .. controls (318.65,142.72) and (317.67,141.74) .. (317.67,140.53) -- cycle ;
\draw  [color={rgb, 255:red, 208; green, 2; blue, 27 }  ,draw opacity=1 ][fill={rgb, 255:red, 208; green, 2; blue, 27 }  ,fill opacity=1 ] (334.26,140.93) .. controls (334.26,139.72) and (335.24,138.74) .. (336.45,138.74) .. controls (337.66,138.74) and (338.64,139.72) .. (338.64,140.93) .. controls (338.64,142.14) and (337.66,143.12) .. (336.45,143.12) .. controls (335.24,143.12) and (334.26,142.14) .. (334.26,140.93) -- cycle ;
\draw  [color={rgb, 255:red, 208; green, 2; blue, 27 }  ,draw opacity=1 ][fill={rgb, 255:red, 208; green, 2; blue, 27 }  ,fill opacity=1 ] (381.26,140.14) .. controls (381.26,138.93) and (382.24,137.95) .. (383.45,137.95) .. controls (384.66,137.95) and (385.64,138.93) .. (385.64,140.14) .. controls (385.64,141.35) and (384.66,142.33) .. (383.45,142.33) .. controls (382.24,142.33) and (381.26,141.35) .. (381.26,140.14) -- cycle ;
\draw  [color={rgb, 255:red, 208; green, 2; blue, 27 }  ,draw opacity=1 ][fill={rgb, 255:red, 208; green, 2; blue, 27 }  ,fill opacity=1 ] (302.66,125.52) .. controls (302.66,124.32) and (303.64,123.34) .. (304.85,123.34) .. controls (306.06,123.34) and (307.04,124.32) .. (307.04,125.52) .. controls (307.04,126.73) and (306.06,127.71) .. (304.85,127.71) .. controls (303.64,127.71) and (302.66,126.73) .. (302.66,125.52) -- cycle ;
\draw  [color={rgb, 255:red, 208; green, 2; blue, 27 }  ,draw opacity=1 ][fill={rgb, 255:red, 208; green, 2; blue, 27 }  ,fill opacity=1 ] (302.66,108.54) .. controls (302.66,107.33) and (303.64,106.35) .. (304.85,106.35) .. controls (306.06,106.35) and (307.04,107.33) .. (307.04,108.54) .. controls (307.04,109.75) and (306.06,110.73) .. (304.85,110.73) .. controls (303.64,110.73) and (302.66,109.75) .. (302.66,108.54) -- cycle ;
\draw  [color={rgb, 255:red, 208; green, 2; blue, 27 }  ,draw opacity=1 ][fill={rgb, 255:red, 208; green, 2; blue, 27 }  ,fill opacity=1 ] (302.66,58.38) .. controls (302.66,57.17) and (303.64,56.19) .. (304.85,56.19) .. controls (306.06,56.19) and (307.04,57.17) .. (307.04,58.38) .. controls (307.04,59.59) and (306.06,60.57) .. (304.85,60.57) .. controls (303.64,60.57) and (302.66,59.59) .. (302.66,58.38) -- cycle ;
\draw [color={rgb, 255:red, 155; green, 155; blue, 155 }  ,draw opacity=1 ]   (319.86,140.53) -- (319.86,30.91) ;
\draw [color={rgb, 255:red, 155; green, 155; blue, 155 }  ,draw opacity=1 ]   (336.45,140.93) -- (336.45,31.3) ;
\draw [color={rgb, 255:red, 155; green, 155; blue, 155 }  ,draw opacity=1 ]   (383.45,140.14) -- (383.45,30.51) ;
\draw [color={rgb, 255:red, 155; green, 155; blue, 155 }  ,draw opacity=1 ]   (414.4,58.38) -- (304.85,58.38) ;
\draw [color={rgb, 255:red, 155; green, 155; blue, 155 }  ,draw opacity=1 ]   (414.4,108.54) -- (304.85,108.54) ;
\draw [color={rgb, 255:red, 155; green, 155; blue, 155 }  ,draw opacity=1 ]   (414.4,125.52) -- (304.85,125.52) ;
\draw  [color={rgb, 255:red, 74; green, 144; blue, 226 }  ,draw opacity=1 ][fill={rgb, 255:red, 74; green, 144; blue, 226 }  ,fill opacity=0.17 ] (312.45,114.43) .. controls (307.83,94.97) and (325.08,66.33) .. (350.99,50.46) .. controls (376.89,34.58) and (401.65,37.48) .. (406.27,56.93) .. controls (410.9,76.39) and (393.65,105.03) .. (367.74,120.9) .. controls (341.83,136.78) and (317.08,133.88) .. (312.45,114.43) -- cycle ;
\draw [color={rgb, 255:red, 74; green, 144; blue, 226 }  ,draw opacity=1 ]   (373.71,21.16) -- (367.8,50.01) ;
\draw [shift={(367.39,51.97)}, rotate = 281.59] [color={rgb, 255:red, 74; green, 144; blue, 226 }  ,draw opacity=1 ][line width=0.75]    (10.93,-3.29) .. controls (6.95,-1.4) and (3.31,-0.3) .. (0,0) .. controls (3.31,0.3) and (6.95,1.4) .. (10.93,3.29)   ;
\draw   (26.42,180.74) -- (135.43,180.74) -- (135.43,289.75) -- (26.42,289.75) -- cycle ;
\draw  [color={rgb, 255:red, 208; green, 2; blue, 27 }  ,draw opacity=1 ][fill={rgb, 255:red, 208; green, 2; blue, 27 }  ,fill opacity=1 ] (39.23,290.1) .. controls (39.23,288.89) and (40.21,287.91) .. (41.42,287.91) .. controls (42.63,287.91) and (43.61,288.89) .. (43.61,290.1) .. controls (43.61,291.31) and (42.63,292.29) .. (41.42,292.29) .. controls (40.21,292.29) and (39.23,291.31) .. (39.23,290.1) -- cycle ;
\draw  [color={rgb, 255:red, 208; green, 2; blue, 27 }  ,draw opacity=1 ][fill={rgb, 255:red, 208; green, 2; blue, 27 }  ,fill opacity=1 ] (55.82,290.49) .. controls (55.82,289.28) and (56.8,288.3) .. (58.01,288.3) .. controls (59.22,288.3) and (60.2,289.28) .. (60.2,290.49) .. controls (60.2,291.7) and (59.22,292.68) .. (58.01,292.68) .. controls (56.8,292.68) and (55.82,291.7) .. (55.82,290.49) -- cycle ;
\draw  [color={rgb, 255:red, 208; green, 2; blue, 27 }  ,draw opacity=1 ][fill={rgb, 255:red, 208; green, 2; blue, 27 }  ,fill opacity=1 ] (102.82,289.7) .. controls (102.82,288.49) and (103.8,287.51) .. (105.01,287.51) .. controls (106.22,287.51) and (107.2,288.49) .. (107.2,289.7) .. controls (107.2,290.91) and (106.22,291.89) .. (105.01,291.89) .. controls (103.8,291.89) and (102.82,290.91) .. (102.82,289.7) -- cycle ;
\draw  [color={rgb, 255:red, 208; green, 2; blue, 27 }  ,draw opacity=1 ][fill={rgb, 255:red, 208; green, 2; blue, 27 }  ,fill opacity=1 ] (24.22,275.09) .. controls (24.22,273.88) and (25.2,272.9) .. (26.41,272.9) .. controls (27.62,272.9) and (28.6,273.88) .. (28.6,275.09) .. controls (28.6,276.3) and (27.62,277.28) .. (26.41,277.28) .. controls (25.2,277.28) and (24.22,276.3) .. (24.22,275.09) -- cycle ;
\draw  [color={rgb, 255:red, 208; green, 2; blue, 27 }  ,draw opacity=1 ][fill={rgb, 255:red, 208; green, 2; blue, 27 }  ,fill opacity=1 ] (24.22,258.1) .. controls (24.22,256.9) and (25.2,255.92) .. (26.41,255.92) .. controls (27.62,255.92) and (28.6,256.9) .. (28.6,258.1) .. controls (28.6,259.31) and (27.62,260.29) .. (26.41,260.29) .. controls (25.2,260.29) and (24.22,259.31) .. (24.22,258.1) -- cycle ;
\draw  [color={rgb, 255:red, 208; green, 2; blue, 27 }  ,draw opacity=1 ][fill={rgb, 255:red, 208; green, 2; blue, 27 }  ,fill opacity=1 ] (24.22,207.94) .. controls (24.22,206.73) and (25.2,205.75) .. (26.41,205.75) .. controls (27.62,205.75) and (28.6,206.73) .. (28.6,207.94) .. controls (28.6,209.15) and (27.62,210.13) .. (26.41,210.13) .. controls (25.2,210.13) and (24.22,209.15) .. (24.22,207.94) -- cycle ;
\draw [color={rgb, 255:red, 155; green, 155; blue, 155 }  ,draw opacity=0.61 ]   (41.42,290.1) -- (41.42,180.47) ;
\draw [color={rgb, 255:red, 155; green, 155; blue, 155 }  ,draw opacity=0.61 ]   (58.01,290.49) -- (58.01,180.87) ;
\draw [color={rgb, 255:red, 155; green, 155; blue, 155 }  ,draw opacity=0.61 ]   (105.01,289.7) -- (105.01,180.08) ;
\draw [color={rgb, 255:red, 155; green, 155; blue, 155 }  ,draw opacity=0.61 ]   (135.96,207.94) -- (26.41,207.94) ;
\draw [color={rgb, 255:red, 155; green, 155; blue, 155 }  ,draw opacity=0.61 ]   (135.96,275.09) -- (26.41,275.09) ;
\draw  [color={rgb, 255:red, 74; green, 144; blue, 226 }  ,draw opacity=1 ] (34.02,263.99) .. controls (29.39,244.54) and (46.64,215.9) .. (72.55,200.02) .. controls (98.45,184.14) and (123.21,187.04) .. (127.83,206.5) .. controls (132.46,225.95) and (115.21,254.59) .. (89.3,270.47) .. controls (63.4,286.34) and (38.64,283.44) .. (34.02,263.99) -- cycle ;
\draw   (165.19,182.84) -- (274.2,182.84) -- (274.2,291.86) -- (165.19,291.86) -- cycle ;
\draw  [color={rgb, 255:red, 208; green, 2; blue, 27 }  ,draw opacity=1 ][fill={rgb, 255:red, 208; green, 2; blue, 27 }  ,fill opacity=1 ] (178,292.2) .. controls (178,291) and (178.98,290.02) .. (180.19,290.02) .. controls (181.4,290.02) and (182.38,291) .. (182.38,292.2) .. controls (182.38,293.41) and (181.4,294.39) .. (180.19,294.39) .. controls (178.98,294.39) and (178,293.41) .. (178,292.2) -- cycle ;
\draw  [color={rgb, 255:red, 208; green, 2; blue, 27 }  ,draw opacity=1 ][fill={rgb, 255:red, 208; green, 2; blue, 27 }  ,fill opacity=1 ] (194.59,292.6) .. controls (194.59,291.39) and (195.57,290.41) .. (196.78,290.41) .. controls (197.99,290.41) and (198.97,291.39) .. (198.97,292.6) .. controls (198.97,293.81) and (197.99,294.79) .. (196.78,294.79) .. controls (195.57,294.79) and (194.59,293.81) .. (194.59,292.6) -- cycle ;
\draw  [color={rgb, 255:red, 208; green, 2; blue, 27 }  ,draw opacity=1 ][fill={rgb, 255:red, 208; green, 2; blue, 27 }  ,fill opacity=1 ] (241.59,291.81) .. controls (241.59,290.6) and (242.57,289.62) .. (243.78,289.62) .. controls (244.99,289.62) and (245.97,290.6) .. (245.97,291.81) .. controls (245.97,293.02) and (244.99,294) .. (243.78,294) .. controls (242.57,294) and (241.59,293.02) .. (241.59,291.81) -- cycle ;
\draw  [color={rgb, 255:red, 208; green, 2; blue, 27 }  ,draw opacity=1 ][fill={rgb, 255:red, 208; green, 2; blue, 27 }  ,fill opacity=1 ] (162.99,277.19) .. controls (162.99,275.99) and (163.97,275.01) .. (165.18,275.01) .. controls (166.39,275.01) and (167.37,275.99) .. (167.37,277.19) .. controls (167.37,278.4) and (166.39,279.38) .. (165.18,279.38) .. controls (163.97,279.38) and (162.99,278.4) .. (162.99,277.19) -- cycle ;
\draw  [color={rgb, 255:red, 208; green, 2; blue, 27 }  ,draw opacity=1 ][fill={rgb, 255:red, 208; green, 2; blue, 27 }  ,fill opacity=1 ] (162.99,260.21) .. controls (162.99,259) and (163.97,258.02) .. (165.18,258.02) .. controls (166.39,258.02) and (167.37,259) .. (167.37,260.21) .. controls (167.37,261.42) and (166.39,262.4) .. (165.18,262.4) .. controls (163.97,262.4) and (162.99,261.42) .. (162.99,260.21) -- cycle ;
\draw  [color={rgb, 255:red, 208; green, 2; blue, 27 }  ,draw opacity=1 ][fill={rgb, 255:red, 208; green, 2; blue, 27 }  ,fill opacity=1 ] (162.99,210.05) .. controls (162.99,208.84) and (163.97,207.86) .. (165.18,207.86) .. controls (166.39,207.86) and (167.37,208.84) .. (167.37,210.05) .. controls (167.37,211.26) and (166.39,212.24) .. (165.18,212.24) .. controls (163.97,212.24) and (162.99,211.26) .. (162.99,210.05) -- cycle ;
\draw [color={rgb, 255:red, 155; green, 155; blue, 155 }  ,draw opacity=0.61 ]   (180.19,292.2) -- (180.19,182.58) ;
\draw [color={rgb, 255:red, 155; green, 155; blue, 155 }  ,draw opacity=0.61 ]   (196.78,292.6) -- (196.78,182.97) ;
\draw [color={rgb, 255:red, 155; green, 155; blue, 155 }  ,draw opacity=0.61 ]   (243.78,291.81) -- (243.78,182.18) ;
\draw [color={rgb, 255:red, 155; green, 155; blue, 155 }  ,draw opacity=0.61 ]   (274.73,210.05) -- (165.18,210.05) ;
\draw [color={rgb, 255:red, 155; green, 155; blue, 155 }  ,draw opacity=0.61 ]   (274.73,260.21) -- (165.18,260.21) ;
\draw [color={rgb, 255:red, 155; green, 155; blue, 155 }  ,draw opacity=0.61 ]   (274.73,277.19) -- (165.18,277.19) ;
\draw [color={rgb, 255:red, 208; green, 2; blue, 27 }  ,draw opacity=1 ][line width=2.25]    (104.8,190.22) -- (104.8,258.63) ;
\draw [color={rgb, 255:red, 208; green, 2; blue, 27 }  ,draw opacity=1 ][line width=2.25]    (57.62,212.39) -- (57.62,282.25) ;
\draw [color={rgb, 255:red, 208; green, 2; blue, 27 }  ,draw opacity=1 ][line width=2.25]    (40.85,232.84) -- (40.85,277.01) ;
\draw [color={rgb, 255:red, 208; green, 2; blue, 27 }  ,draw opacity=1 ][line width=2.25]    (80.03,276.22) -- (40.85,277.01) ;
\draw [color={rgb, 255:red, 208; green, 2; blue, 27 }  ,draw opacity=1 ][line width=2.25]    (104.8,259.05) -- (33.64,259.05) ;
\draw [color={rgb, 255:red, 208; green, 2; blue, 27 }  ,draw opacity=1 ][line width=2.25]    (62.47,207.65) -- (128,207.65) ;
\draw [color={rgb, 255:red, 208; green, 2; blue, 27 }  ,draw opacity=1 ][line width=2.25]    (244.3,191.72) -- (244.3,260.13) ;
\draw [color={rgb, 255:red, 208; green, 2; blue, 27 }  ,draw opacity=1 ][line width=2.25]    (197.12,213.89) -- (197.12,283.75) ;
\draw [color={rgb, 255:red, 208; green, 2; blue, 27 }  ,draw opacity=1 ][line width=2.25]    (180.35,234.34) -- (180.35,278.51) ;
\draw [color={rgb, 255:red, 208; green, 2; blue, 27 }  ,draw opacity=1 ][line width=2.25]    (219.53,277.72) -- (180.35,278.51) ;
\draw [color={rgb, 255:red, 208; green, 2; blue, 27 }  ,draw opacity=1 ][line width=2.25]    (244.3,260.55) -- (173.14,260.55) ;
\draw [color={rgb, 255:red, 208; green, 2; blue, 27 }  ,draw opacity=1 ][line width=2.25]    (201.97,209.15) -- (267.5,209.15) ;
\draw [color={rgb, 255:red, 208; green, 2; blue, 27 }  ,draw opacity=1 ]   (219.48,177.5) -- (211.64,206.57) ;
\draw [shift={(211.12,208.5)}, rotate = 285.09] [color={rgb, 255:red, 208; green, 2; blue, 27 }  ,draw opacity=1 ][line width=0.75]    (10.93,-3.29) .. controls (6.95,-1.4) and (3.31,-0.3) .. (0,0) .. controls (3.31,0.3) and (6.95,1.4) .. (10.93,3.29)   ;
\draw [color={rgb, 255:red, 208; green, 2; blue, 27 }  ,draw opacity=1 ]   (305.25,274.5) -- (245.96,249.77) ;
\draw [shift={(244.12,249)}, rotate = 22.64] [color={rgb, 255:red, 208; green, 2; blue, 27 }  ,draw opacity=1 ][line width=0.75]    (10.93,-3.29) .. controls (6.95,-1.4) and (3.31,-0.3) .. (0,0) .. controls (3.31,0.3) and (6.95,1.4) .. (10.93,3.29)   ;

\draw (116.24,141.33) node [anchor=north west][inner sep=0.75pt]    {$V_{1}$};
\draw (2.88,29.15) node [anchor=north west][inner sep=0.75pt]    {$V_{2}$};
\draw (80.26,-0.07) node [anchor=north west][inner sep=0.75pt]    {$V_{1} \times V_{2}$};
\draw (264.25,141.54) node [anchor=north west][inner sep=0.75pt]    {$V_{1}$};
\draw (141.5,27.15) node [anchor=north west][inner sep=0.75pt]    {$V_{2}$};
\draw (220.87,-0.07) node [anchor=north west][inner sep=0.75pt]    {$V_{1} \times V_{2}$};
\draw (168.65,144.17) node [anchor=north west][inner sep=0.75pt]  [font=\footnotesize]  {$\textcolor[rgb]{0.82,0.01,0.11}{a}\textcolor[rgb]{0.82,0.01,0.11}{_{1,1}}$};
\draw (189.6,145.12) node [anchor=north west][inner sep=0.75pt]  [font=\footnotesize]  {$\textcolor[rgb]{0.82,0.01,0.11}{a}\textcolor[rgb]{0.82,0.01,0.11}{_{2,1}}$};
\draw (209.04,143.01) node [anchor=north west][inner sep=0.75pt]    {$\textcolor[rgb]{0.82,0.01,0.11}{\dotsc }$};
\draw (237.44,144.91) node [anchor=north west][inner sep=0.75pt]  [font=\footnotesize]  {$\textcolor[rgb]{0.82,0.01,0.11}{a}\textcolor[rgb]{0.82,0.01,0.11}{_{N,1}}$};
\draw (137.05,117.29) node [anchor=north west][inner sep=0.75pt]  [font=\footnotesize]  {$\textcolor[rgb]{0.82,0.01,0.11}{a}\textcolor[rgb]{0.82,0.01,0.11}{_{1,2}}$};
\draw (137.84,101.49) node [anchor=north west][inner sep=0.75pt]  [font=\footnotesize]  {$\textcolor[rgb]{0.82,0.01,0.11}{a}\textcolor[rgb]{0.82,0.01,0.11}{_{2,2}}$};
\draw (142.68,69.92) node [anchor=north west][inner sep=0.75pt]    {$\textcolor[rgb]{0.82,0.01,0.11}{\vdots }$};
\draw (136.73,50.22) node [anchor=north west][inner sep=0.75pt]  [font=\footnotesize]  {$\textcolor[rgb]{0.82,0.01,0.11}{a}\textcolor[rgb]{0.82,0.01,0.11}{_{N,2}}$};
\draw (403.13,141.07) node [anchor=north west][inner sep=0.75pt]    {$V_{1}$};
\draw (282.38,26.68) node [anchor=north west][inner sep=0.75pt]    {$V_{2}$};
\draw (352.9,6.06) node [anchor=north west][inner sep=0.75pt]  [color={rgb, 255:red, 74; green, 144; blue, 226 }  ,opacity=1 ]  {$f\in \mathcal{F}$};
\draw (265.57,292.74) node [anchor=north west][inner sep=0.75pt]    {$V_{1}$};
\draw (143.92,177.4) node [anchor=north west][inner sep=0.75pt]    {$V_{2}$};
\draw (279.58,193.24) node [anchor=north west][inner sep=0.75pt]  [font=\scriptsize]  {$ \begin{array}{l}
\text{Need to (approximately) recover }\\
\textcolor[rgb]{0.29,0.56,0.89}{f}\text{ given these slices for the sample} \ \\
\textcolor[rgb]{0.82,0.01,0.11}{a}\textcolor[rgb]{0.82,0.01,0.11}{_{1,1}}\textcolor[rgb]{0.82,0.01,0.11}{,\dotsc ,a}\textcolor[rgb]{0.82,0.01,0.11}{_{N,1}}\textcolor[rgb]{0.82,0.01,0.11}{;a}\textcolor[rgb]{0.82,0.01,0.11}{_{1,2}}\textcolor[rgb]{0.82,0.01,0.11}{,\dotsc ,a}\textcolor[rgb]{0.82,0.01,0.11}{_{N,2}{}}\textcolor[rgb]{0.82,0.01,0.11}{}
\end{array}$};
\draw (145.16,222.27) node [anchor=north west][inner sep=0.75pt]    {$\textcolor[rgb]{0.82,0.01,0.11}{\vdots }$};
\draw (168.75,297) node [anchor=north west][inner sep=0.75pt]  [font=\footnotesize]  {$\textcolor[rgb]{0.82,0.01,0.11}{a}\textcolor[rgb]{0.82,0.01,0.11}{_{1,1}}$};
\draw (191.2,297.45) node [anchor=north west][inner sep=0.75pt]  [font=\footnotesize]  {$\textcolor[rgb]{0.82,0.01,0.11}{a}\textcolor[rgb]{0.82,0.01,0.11}{_{2,1}}$};
\draw (212.65,298.34) node [anchor=north west][inner sep=0.75pt]    {$\textcolor[rgb]{0.82,0.01,0.11}{\dotsc }$};
\draw (241.05,298.24) node [anchor=north west][inner sep=0.75pt]  [font=\footnotesize]  {$\textcolor[rgb]{0.82,0.01,0.11}{a}\textcolor[rgb]{0.82,0.01,0.11}{_{N,1}}$};
\draw (137.65,267.61) node [anchor=north west][inner sep=0.75pt]  [font=\footnotesize]  {$\textcolor[rgb]{0.82,0.01,0.11}{a}\textcolor[rgb]{0.82,0.01,0.11}{_{1,2}}$};
\draw (138.05,251.82) node [anchor=north west][inner sep=0.75pt]  [font=\footnotesize]  {$\textcolor[rgb]{0.82,0.01,0.11}{a}\textcolor[rgb]{0.82,0.01,0.11}{_{2,2}}$};
\draw (136.55,201.26) node [anchor=north west][inner sep=0.75pt]  [font=\footnotesize]  {$\textcolor[rgb]{0.82,0.01,0.11}{a}\textcolor[rgb]{0.82,0.01,0.11}{_{N,2}}$};
\draw (308.65,143.67) node [anchor=north west][inner sep=0.75pt]  [font=\footnotesize]  {$\textcolor[rgb]{0.82,0.01,0.11}{a}\textcolor[rgb]{0.82,0.01,0.11}{_{1,1}}$};
\draw (329.6,144.62) node [anchor=north west][inner sep=0.75pt]  [font=\footnotesize]  {$\textcolor[rgb]{0.82,0.01,0.11}{a}\textcolor[rgb]{0.82,0.01,0.11}{_{2,1}}$};
\draw (349.04,142.51) node [anchor=north west][inner sep=0.75pt]    {$\textcolor[rgb]{0.82,0.01,0.11}{\dotsc }$};
\draw (377.44,144.41) node [anchor=north west][inner sep=0.75pt]  [font=\footnotesize]  {$\textcolor[rgb]{0.82,0.01,0.11}{a}\textcolor[rgb]{0.82,0.01,0.11}{_{N,1}}$};
\draw (277.05,116.79) node [anchor=north west][inner sep=0.75pt]  [font=\footnotesize]  {$\textcolor[rgb]{0.82,0.01,0.11}{a}\textcolor[rgb]{0.82,0.01,0.11}{_{1,2}}$};
\draw (277.84,100.99) node [anchor=north west][inner sep=0.75pt]  [font=\footnotesize]  {$\textcolor[rgb]{0.82,0.01,0.11}{a}\textcolor[rgb]{0.82,0.01,0.11}{_{2,2}}$};
\draw (282.68,69.42) node [anchor=north west][inner sep=0.75pt]    {$\textcolor[rgb]{0.82,0.01,0.11}{\vdots }$};
\draw (276.73,49.22) node [anchor=north west][inner sep=0.75pt]  [font=\footnotesize]  {$\textcolor[rgb]{0.82,0.01,0.11}{a}\textcolor[rgb]{0.82,0.01,0.11}{_{N,2}}$};
\draw (33.15,295.1) node [anchor=north west][inner sep=0.75pt]  [font=\footnotesize]  {$\textcolor[rgb]{0.82,0.01,0.11}{a}\textcolor[rgb]{0.82,0.01,0.11}{_{1,1}}$};
\draw (54.1,296.05) node [anchor=north west][inner sep=0.75pt]  [font=\footnotesize]  {$\textcolor[rgb]{0.82,0.01,0.11}{a}\textcolor[rgb]{0.82,0.01,0.11}{_{2,1}}$};
\draw (75.54,296.94) node [anchor=north west][inner sep=0.75pt]    {$\textcolor[rgb]{0.82,0.01,0.11}{\dotsc }$};
\draw (101.94,295.84) node [anchor=north west][inner sep=0.75pt]  [font=\footnotesize]  {$\textcolor[rgb]{0.82,0.01,0.11}{a}\textcolor[rgb]{0.82,0.01,0.11}{_{N,1}}$};
\draw (-0.45,266.22) node [anchor=north west][inner sep=0.75pt]  [font=\footnotesize]  {$\textcolor[rgb]{0.82,0.01,0.11}{a}\textcolor[rgb]{0.82,0.01,0.11}{_{1,2}}$};
\draw (0.34,250.42) node [anchor=north west][inner sep=0.75pt]  [font=\footnotesize]  {$\textcolor[rgb]{0.82,0.01,0.11}{a}\textcolor[rgb]{0.82,0.01,0.11}{_{2,2}}$};
\draw (5.18,216.85) node [anchor=north west][inner sep=0.75pt]    {$\textcolor[rgb]{0.82,0.01,0.11}{\vdots }$};
\draw (-0.77,198.65) node [anchor=north west][inner sep=0.75pt]  [font=\footnotesize]  {$\textcolor[rgb]{0.82,0.01,0.11}{a}\textcolor[rgb]{0.82,0.01,0.11}{_{N,2}}$};
\draw (186.96,164.32) node [anchor=north west][inner sep=0.75pt]  [font=\scriptsize,color={rgb, 255:red, 208; green, 2; blue, 27 }  ,opacity=1 ]  {$f_{a_{N,2}} -\text{slice of }\textcolor[rgb]{0.29,0.56,0.89}{f}\text{ at } a_{N,2}$};
\draw (286.46,277.32) node [anchor=north west][inner sep=0.75pt]  [font=\scriptsize,color={rgb, 255:red, 208; green, 2; blue, 27 }  ,opacity=1 ]  {$f_{a_{N,1}} -\text{slice of }\textcolor[rgb]{0.29,0.56,0.89}{f}\text{ at } a_{N,1}$};

\end{tikzpicture}

\begin{fac}\cite[Theorem 5]{KotaPAC}\label{fac: Kota one dir} Every $\PAC_n$-learnable class  has finite $\VC_n$-dimension. 
\end{fac}
\noindent The converse implication remained open.

\section{Packing lemma for families of finite VC$_k$ dimension}

The following is a classical \emph{packing lemma} of Haussler:
\begin{fac}\cite{haussler1995sphere} 
	For every $d$ and $\varepsilon >0$ there is $N = N(d,\varepsilon)$ so that: 
 if $\left(V,\mu\right)$ is a probability
space and $\mathcal{F}$ is a family of measurable subsets of $V$
with $VC\left(\mathcal{F}\right)\leq d$, then there are some $S_1, \ldots, S_N \in \mathcal{F}$ so that  every $S \in \mathcal{F}$ satisfies $\mu \left( S \Delta S_i \right) \leq \varepsilon$ for some $i$.
\end{fac}

In other words, there is a bounded (in terms of $d$ and $\varepsilon$) number of sets in $\mathcal{F}$ so that \emph{every} set in $\mathcal{F}$ is $\varepsilon$-close to one of them. One of the main results of \cite{chernikov2020hypergraph} is a higher arity generalization of Haussler's packing lemma from $\VC$-dimension to $\VC_k$-dimension. E.g.~for $k=2$, given a family $\mathcal{F}$ of subsets of $V_1 \times V_2$ of finite $\VC_2$-dimension, we can no longer expect that every set in the family is $\varepsilon$-close (with respect to the product measure $\mu_1 \otimes \mu_2$) to one of a bounded number of sets from $\mathcal{F}$. What we get instead is that there is some $N = N(d,\varepsilon)$ and sets $S_1, \ldots, S_N \in \mathcal{F}$, so that for every $S \in \mathcal{F}$ we have $\mu_1 \otimes \mu_2(S \triangle D) \leq \varepsilon$ for some $D$ given by a Boolean combination of $S_1, \ldots, S_N$ and at most $N$ cylinders over smaller arity slices of the form $S_{b_i}\times V_2$ or $(S_i)_{b_i} \times V_2$ (and $V_1 \times S_{a_i}$ or $V_1 \times (S_i)_{a_i}$)   for some $a_i \in V_1, b_i \in V_2$ that may vary with $S$. We state it now for general $k$:

\begin{definition}
	Given $S \subseteq V_1 \ttimes V_k$, $u \subseteq [k] = \{1, \ldots, k\}$ and $\bar{a} = (a_i : i \in u) \in \prod_{i \in u} V_i$, we let $S'_{\bar{a}} := \{(v_i : i \in [k]) \in S : \bigwedge_{i \in u} v_i = a_i  \}$ and $S_{\bar{a}} := \pi_{[k] \setminus u}(S'_{\bar{a}})$ --- so $S_{\bar{a}} \subseteq \prod_{i \in [k] \setminus u} V_i$ is the $(k-|u|)$-ary fiber of $S$ at $\bar{a}$.
\end{definition}

\begin{fac}\label{fac: finite VCk-dim implies approx bounded}\cite[Proposition 5.5]{chernikov2020hypergraph}
For any $k, d$ and $\varepsilon>0$ there is $N = N(k,d, \varepsilon)$ satisfying the following.
	Let $(V_i, \mu_i)$ be probability spaces and $\mathcal{F}$ a family of subsets of $V_1 \ttimes V_k$ with $\VC_k(\mathcal{F}) \leq d$. Then there exist $S_1, \ldots, S_N \in \mathcal{F}$ such that for every $S \in \mathcal{F}$ we have $\mu_1 \otimes \ldots \otimes \mu_k (S \Delta D) \leq \varepsilon$ for some $D$ given by a Boolean combination of $S_1, \ldots, S_N$ and $\leq N$ sets given by $\leq (k-1)$-ary fibers of $S, S_1, \ldots, S_N$, i.e.~sets of the form 
	$$\{(v_1, \ldots, v_k) \in V_1 \ttimes V_k : (v_i : i \in [k] \setminus u) \in  S_{\bar{a}}\}$$
	 for some $u \subseteq [k], |u| \geq 1$ and $\bar{a} \in \prod_{i \in u}V_i$ (which may depend on $S$).

 \end{fac}

\begin{remark}
	Note that for $k = 1$, the only $0$-ary fibers of $S$ are $\emptyset$ and $V$, so we (qualitatively) recover the classical Haussler's packing lemma.  Our result in \cite[Proposition 5.5]{chernikov2020hypergraph} is in fact more general, for families of $[0,1]$-valued functions instead of just $\{0,1\}$-valued functions. 
\end{remark}

%
%

\section{Equivalence of finite $\VC_k$-dimension to $\PAC_k$-learning}\label{sec: VCk PACk}

We will use the weak law of large numbers, in the following simple form:
\begin{fac}\label{fac: weak law of large numbers}
	For every $\varepsilon, \delta \in \mathbb{R}_{>0}$ and $k \in \mathbb{N}_{\geq 1}$ there exists $N = N(\varepsilon, \delta, k)$ satisfying the following. For any probability space $(\Omega, \mathcal{B}, \mu)$ and any $S_1, \ldots, S_k \in \B$ we have:
	\begin{gather*}
		\mu^{N} \left( (x_1, \ldots, x_N) \in \Omega^{n} : \bigvee_{i=1}^{k} \left(  \left \lvert  \frac{1}{N} \sum_{j=1}^{N} \chi_{S_i}(x_j) - \mu(S_i) \right \rvert \geq \varepsilon \right) \right) \leq \delta,
	\end{gather*}
	where $\mu^n$ is the product measure.
	\end{fac}

	In  \cite[Lemma 5.9]{chernikov2020hypergraph}, we in fact proved a stronger version of the packing lemma for $\VC_k$ dimension demonstrating that for every set $S \in \mathcal{F}$ in the family, there is not just one (as stated in Fact \ref{fac: finite VCk-dim implies approx bounded}) but a \emph{positive measure} set of smaller arity fibers giving an approximation to $S$ within $\varepsilon$ (more precisely, we have proved that if $\mathcal{F}$ satisfies the packing lemma in  Fact \ref{fac: finite VCk-dim implies approx bounded}, then it also satisfies the packing lemma in Fact \ref{fac: packing with pos measure}):
	
\begin{fac}\cite[Lemma 5.9]{chernikov2020hypergraph}\label{fac: packing with pos measure}
For any $k, d$ and $\varepsilon>0$ there is $N = N(k,d, \varepsilon)$ and $\rho = \rho(k,d,\varepsilon) >0$ satisfying the following.

	Let $(V_i, \mu_i)$ be probability spaces and $\mathcal{F}$ a family of subsets of $V_1 \ttimes V_k$ with $\VC_k(\mathcal{F}) \leq d$. Then there exist $S_1, \ldots, S_N \in \mathcal{F}$ such that for every $S \in \mathcal{F}$ there is a set $A_{S} \subseteq (V_1 \times \ldots \times V_k)^N$ with $(\mu_1 \otimes \ldots \otimes \mu_k)^{\otimes N}(A_S) \geq \rho$ so that for \emph{every} $(\bar{a}_1, \ldots, \bar{a}_N) \in A_S$ we have $\mu_1 \otimes \ldots \otimes \mu_k (S \Delta D) \leq \varepsilon$ for some $D$ given by a Boolean combination of $S_1, \ldots, S_N$ and $\leq N$ sets given by $\leq (k-1)$-ary fibers of $S, S_1, \ldots, S_N$ of the form 
	$$\{(v_1, \ldots, v_k) \in V_1 \ttimes V_k : (v_i : i \in [k] \setminus u) \in  S_{\bar{a}_j}\}$$
	 for some $u \subseteq [k], |u| \geq 1$ and $1 \leq j \leq N$.
 \end{fac}

\begin{remark}
	It is not stated in \cite[Lemma 5.9]{chernikov2020hypergraph} explicitly that $\rho$ depends only on $d$ and $\varepsilon$, but it follows by compactness since the result is proved for all probability spaces simultaneously (see \cite[Section 9.3]{chernikov2020hypergraph}).
\end{remark}

\begin{remark}\label{rem: VC_k dim and packing lemma equiv}
	In fact, \cite{chernikov2020hypergraph} established the equivalence of the packing lemma in Fact \ref{fac: finite VCk-dim implies approx bounded} and finite $\VC_k$-dimension for a family of sets $\mathcal{F}$, via the following sequence of implications: $\mathcal{F}$ has finite $\VC_k$-dimension $\Rightarrow$ $\mathcal{F}$  satisfies packing lemma uniformly over all measures (Fact \ref{fac: finite VCk-dim implies approx bounded},  \cite[Proposition 5.5]{chernikov2020hypergraph}) $\Rightarrow$ $\mathcal{F}$ satisfies packing lemma with positive measure of witnesses (Fact \ref{fac: packing with pos measure},  \cite[Lemma 5.9]{chernikov2020hypergraph}) $\Rightarrow$ strong hypergraph regularity holds for the incidence hypergraph of $\mathcal{F}$ uniformly over all measures (\cite[Theorem 6.6]{chernikov2020hypergraph})   $\Rightarrow$ $\mathcal{F}$ has finite $\VC_k$-dimension (\cite[Corollary 7.3]{chernikov2020hypergraph}).
\end{remark}

\begin{theorem}\label{thm: pack implies PAC}
Every class $\mathcal{F}$ of finite $\VC_k$ dimension is properly $\PAC_k$-learnable.
\end{theorem}

\begin{proof}

Let $k$ be fixed, and $d \in \mathbb{N}$ and $\varepsilon,\delta \in \mathbb{R}_{>0}$ be given. 
Assume we are given arbitrary $V = V_1 \times \ldots \times V_k$, $\mathcal{F} \subseteq \mathcal{P}(V)$ with $\VC_k(\mathcal{F}) \leq d$, $\B_i \subseteq \mathcal{P}(V_i)$ are  $\sigma$-algebras and $\mu_i$ probability measures on $\B_i$ forming a $k$-partite graded probability space. Let $\mu := \mu_1 \otimes \ldots \otimes \mu_k$.
	
	 By the packing lemma for $\VC_k$ dimension (Fact \ref{fac: packing with pos measure}), there exist  $N_1 = N_1(d, \varepsilon)$, $\rho = \rho(d, \varepsilon) > 0$ and $S_1, \ldots, S_{N_1} \in \mathcal{F}$ so that: for every $S \in \mathcal{F}$, the $\mu^{\otimes N_1}$-measure of the set of tuples $\bar{x}_1 \in V^{N_1}$ so that 
	\begin{gather}
		\mu(S \triangle D) \leq \frac{\varepsilon}{6} \textrm{ for some } D \textrm{ a Boolean combination of  } \label{eq: approx pos meas fib}\\
		S_1, \ldots, S_{N_1} \textrm{ and some } \leq (k-1)\textrm{-ary } \bar{x}_1 \textrm{-fibers of } S, S_1, \ldots, S_{N_1} \nonumber
	\end{gather}
	is at least $\rho$.
	
	We  can amplify positive measure of such $\bar{x}_1$ to measure arbitrarily close to $1$. Indeed, as $\rho > 0$,  we can choose $\ell = \ell(\rho, \delta') = \ell(d, \varepsilon, \delta)$ be so that $(1-\rho)^{\ell} \leq \delta'$.
  As $\mu^{\otimes \ell N_1}$ extends the product measure $\mu^{\times \ell N_1}$, it follows that for each $S \in \mathcal{F}$, the $\mu^{\otimes \ell N_1}$-measure of the set of tuples $\bar{x}'_1 = (\bar{x}_{1,1}, \ldots, \bar{x}_{1,\ell}) \in V^{\ell \cdot N_1}$ so that none of $\bar{x}_{1, i} $ for $i \in [\ell]$ satisfies \eqref{eq: approx pos meas fib} is at most $\delta' = \delta'(\delta)$.
	
	By the weak law of large numbers (Fact \ref{fac: weak law of large numbers}), we can choose $N_2 = N_2(\varepsilon, \delta', \ell \cdot N_{1}) = N_2(\varepsilon, \delta, N_1)$ so that for any fixed collection of $\ell \cdot N_{1} + 1$ sets in a probability space, for all but measure $\delta'$ tuples $\bar{x}_2 \in V^{N_2}$, for any Boolean combination $F$ of these sets (there are at most $2^{\ell \cdot N_{1} + 1}$-many), $\mu(F)$ is approximated within $\varepsilon/6$ by the fraction of points from $\bar{x}_2$ that are in $F$.
	
	We define a proper learning function $H$ as follows. Given a tuple $(\bar{x}'_1, \bar{x}_2) \in  V^{\ell \cdot N_1} \times V^{N_2}$,  let $D$ be the (lexicographically) first Boolean combination of $S_1, \ldots, S_{N_1}$ and $\leq (k-1)$-ary $\bar{x}'_1$-fibers of $S$ so that $S \Delta D$ contains at most $(2 \varepsilon/6)$-fraction of the points in the tuple $\bar{x}_2$ (which points from $\bar{x}_2$ are in $S$ can be read off from $0$-ary $\bar{x}_2$-fibers of $S$, hence from $f \restriction_{D_k \left(\left( \bar{x}'_1, \bar{x}_2\right) \right)}$, using Remark \ref{rem: reading fibers from Takeuchi}), if it exists, and an arbitrary set in $\mathcal{F}$ otherwise. And we let $H(f \restriction_{D_k \left(\left( \bar{x}'_1, \bar{x}_2\right) \right)})$ be an arbitrary set  $S' \in \mathcal{F}$ (e.g.~the first set with respect to some fixed well-ordering of $\mathcal{F}$) so that $\mu(D \triangle S') \leq \frac{3 \varepsilon}{6}$, if it exists, and an arbitrary set in $\mathcal{F}$ otherwise.

	Fix any $S \in \mathcal{F}$. Now we argue by Fubini. Fix $\bar{x}'_1 \in V^{N'_1}$, where $N'_1 := \ell \cdot N_1$, so that at least one  of $\bar{x}_{1, i} $ for $i \in [\ell]$ satisfies \eqref{eq: approx pos meas fib} (all but measure $\delta'$ of $\bar{x}'_1 \in V^{N'_1}$ satisfy this). 	
	By the choice of $N_2$, for all but  measure $\delta'$ of $\bar{x}_2 \in V^{N_2}$, the fraction of points from $\bar{x}_2$ that are in $S \Delta D$ is within $\varepsilon/6$ of $\mu(S \Delta D)$, for all Boolean combinations $D$ simultaneously. In particular, there is a Boolean combination $D_0$ containing at most $\frac{2 \varepsilon}{6}$-fraction of points from $\bar{x}_2$, so let $D_1$ be the lexicographically first such Boolean combination. As $\mu(S \Delta D_1)$ is also approximated within $\varepsilon/6$ by $\bar{x}_2$, we must have $\mu(S \triangle D_1) \leq \frac{3 \varepsilon}{6}$. By definition, $H$ then returns some set $S' \in \mathcal{F}$ with $\mu(S' \triangle D_1) \leq \frac{3 \varepsilon}{6}$ (existence of $S$ and $D_1$ guarantees that we are in the first case of the definition of $H$), hence $\mu(S \triangle S') \leq \varepsilon$ --- as wanted.
	
	It follows by Fubini that the set of tuples $(\bar{x}'_1, \bar{x}_2) \in V^{N'_1} \times V^{N_2}$ 	for which the learning function $H$ gives an approximation of $S$ within $\varepsilon$  has measure $\geq (1-\delta') \cdot (1-\delta') \geq 1-\delta$, assuming $\delta'$ small enough with respect to $\delta$. So $\mathcal{F}$ is properly $\PAC_k$-learnable.
 	\end{proof}

Combining  Theorem \ref{thm: pack implies PAC} with Fact \ref{fac: Kota one dir} and Remark \ref{rem: VC_k dim and packing lemma equiv}, we thus obtain:
\begin{cor}\label{cor: everything equiv}
	The following are equivalent for a class $\mathcal{F}$ of subsets of $V_1 \times \ldots \times V_k$:
	\begin{enumerate}
		\item $\mathcal{F}$ has finite $\VC_k$-dimension;
		\item $\mathcal{F}$ satisfies the packing lemma (in the sense of Fact \ref{fac: finite VCk-dim implies approx bounded});
		\item $\mathcal{F}$ is $\PAC_k$-learnable (in the sense of Definition \ref{def: PAC_k learn}).
	\end{enumerate}
\end{cor}

\begin{remark}
	We note that another result from \cite[Theorem 10.7]{chernikov2020hypergraph}, in the context of real valued functions, demonstrates that if $f: V_1 \times \ldots \times V_{k+1} \times V_{k+2} \to [0,1]$ is such that $f_z$ has uniformly bounded $\VC_k$-dimension for all $z \in V_{k+2}$ (in the sense of \cite[Definition 3.11]{{chernikov2020hypergraph}}), then $g:  V_1 \times \ldots \times V_{k+1} \to [0,1]$ defined by $g(\bar{x}) = \int_{z \in V_{k+2}} f(\bar{x}, z) d \mu_{k+2}(z)$ also has finite $\VC_k$ dimension (finer questions of this type are studied in \cite{chernikov2025averages}). In the case $k=1$ this generalizes a theorem of Ben Yaacov \cite{yaacov2009continuous}, and implies PAC learnability of such functions as studied in \cite{anderson2025learnable}. We expect that using the results here and in \cite{chernikov2020hypergraph} this generalizes to $\VC_k$ dimension and PAC$_k$-learning.
\end{remark}

\section{Slice-wise packing lemma and hypergraph regularity}\label{sec: slicewise pack}

In \cite{chernikov2020hypergraph}, we also extended the definition of $\VC_k$-dimension to products of arity higher than $k+1$. It is more convenient to formulate it for families of sets given by the slices of hypergraphs:
\begin{definition}\label{def: VCd for higher arity}
  Let $k<k'\in\mathbb{N}$ be arbitrary. We say that a $k'$-ary relation $E\subseteq V_1\ttimes V_{k'}$ has \emph{slice-wise $\VC_k$-dimension $\leq d$} if for any $I\subseteq[k']$ with $|I|=k'-(k+1)$ and any $b\in V_I$, the relation $E_b$ (i.e.~the fiber of $E$ with the coordinates in $I$ fixed by the elements of the tuple $b$, viewed as a $(k+1)$-ary relation on $V_{[k']\setminus I}$) has $\VC_k$-dimension $\leq d$ (in the sense of Definition \ref{def: VCk dim}).

  We write $\VC_k(E)$ for the least $d$ such that slice-wise $\VC_k$-dimension of $E$ is $\leq d$, or $\infty$ if there is no such $d$.
\end{definition}

\begin{remark}
	The terminology ``slice-wise'' was introduced by the authors in the submitted version of the preprint \cite{chernikov2020hypergraph}, and adopted in later versions of \cite{terry2021irregular} where it was initially called ``weak NIP''. 
\end{remark}

\begin{remark}
	The notion of VCN$_k$ dimension later considered in \cite{coregliano2024high, coregliano2025packing} corresponds to slice-wise VC dimension for functions taking finitely many values (\cite{chernikov2020hypergraph, terry2021irregular}), i.e. a special case of the usual VC dimension uniformly bounded for all slices of real valued functions studied in \cite{chernikov2020hypergraph}. This connection was pointed to the authors and reflected in the second version of the preprint \cite{coregliano2024high}, but is again omitted in \cite{coregliano2025packing}.
\end{remark}

The main result of \cite{coregliano2025packing}, relying on the main results of \cite{coregliano2024high}, is a packing lemma for relations of finite slice-wise $\VC_1$-dimension. We point out that, while we did not explicitly state a slice-wise version of our packing lemma for $\VC_k$-dimension (Fact \ref{fac: finite VCk-dim implies approx bounded}) in \cite{chernikov2020hypergraph}, it is implicit in our proof of the hypergraph regularity lemma for hypergraphs of finite slice-wise VC$_k$ dimension there (see the introduction of \cite{chernikov2024perfect} for a brief survey of the area).

Namely, we established the  following slice-wise regularity lemma:
\begin{fac}\cite[Corollary 6.5]{chernikov2020hypergraph}\label{fac: slicewise reg lemma}
	For every $k' > k \geq 1,d$ and $\varepsilon > 0$ there exist some $N = N(k,d, \varepsilon)$ satisfying the following. 
	
	Let $(V_i, \mu_i)$ be probability spaces for $1 \leq i \leq k'$ and $E \subseteq V_1 \times \ldots \times V_{k'}$. For $I \subseteq [k']$, we will write $V_I := \prod_{i \in I} V_i$ and $\mu_I := \bigotimes_{i \in I} \mu_i$.

	Assume that for every $\bar{z} \in V_{[k']\setminus [k+1]}$, the fiber $E_{\bar{z}} \subseteq V_{[k+1]}$ has  $\VC_k$ dimension $\leq d$. Then for every $I \subseteq [k+1], |I| \leq k$ and $1 \leq t \leq N$ there exist sets $S^t_{I} \subseteq \prod_{i \in I \cup ([k']\setminus [k+1]) } V_i$ so that, considering the corresponding cylinders 	$\widetilde{S}^t_{I} := \{(v_1, \ldots, v_{k'}) \in V_{[k']} : (v_i : i \in I \cup  ([k']\setminus [k+1])) \in  S^t_{I} \}$, we have: 
	\begin{enumerate}
	\item 	for all $\bar{z} \in V_{ [k']\setminus [k+1]}$ outside of a set of $\mu_{[k']\setminus [k+1]}$-measure at most $\varepsilon$, we have
	\begin{gather*}
		\mu_{[k+1]} \left(E_{\bar{z}} \triangle \left( \bigcup_{1 \leq t \leq N}  \bigcap_{I \subseteq [k+1], |I| \leq k}(\widetilde{S}^t_{I})_{\bar{z}}\right)  \right) \leq \varepsilon.
	\end{gather*}
\item each $S^t_{I}$ is a Boolean combination of at most $N$ fibers of $E$ with all coordinates outside of $I \cup ([k']\setminus [k+1])$ fixed by some elements, i.e.~sets of the form $E_{\bar{a}} \subseteq \prod_{i \in J} V_i$  for some $J \subseteq I \cup ([k']\setminus [k+1])$ and $\bar{a} \in V_{[k'] \setminus J}$.
	\end{enumerate}
\end{fac}

\begin{remark}
We only cite \cite[Corollary 6.5]{chernikov2020hypergraph} here in the special case of hypergraphs rather than real valued functions (in which case the functions $f^t_I$ there can be taken to be the characteristic functions of some sets $S^t_I$ and weights $\gamma_i$ can be taken in $\{0,1\}$, see \cite[Remark 4.4]{chernikov2020hypergraph}. The uniform bound $N$ in (2) is not explicitly stated in \cite[Corollary 6.5]{chernikov2020hypergraph}, but follows immediately by compactness under the stronger assumption we make here that \emph{all} fibers have bounded $\VC_k$-dimension (applying the non-uniform result in the ultraproduct of counterexamples, see \cite[Corollary 6.9]{chernikov2020hypergraph}).
\end{remark}

In the same way as we have shown that slice-wise hypergraph regularity lemma \emph{uniformly over all measures} implies finite $\VC_k$-dimension (see \cite[
Theorem 7.1]{chernikov2020hypergraph}, applying Fact \ref{fac: slicewise reg lemma} with $\varepsilon < 1$ and taking the measures $\mu_{i}$ for $i \in [k'] \setminus [k+1]$ concentrated on the $i$th coordinate $z_i$ of a single bad fiber $\bar{z} = (z_i : i \in [k']\setminus[k+1]) \in V_{[k'] \setminus [k+1]}$ (as $N$ is independent of the choice of measures), we have:
\begin{cor}\label{cor: for all slices}
	In Fact \ref{fac: slicewise reg lemma}(1), the conclusion can be strengthened from ``for all $\bar{z} \in V_{ [k']\setminus [k+1]}$ outside of a set of $\mu_{[k']\setminus [k+1]}$-measure at most $\varepsilon$'' to ``for all $\bar{z} \in V_{ [k']\setminus [k+1]}$''.
\end{cor}

This immediately gives a slice-wise packing lemma  for $\VC_k$-dimension (similarly to our proof of slice-wise hypergraph regularity in \cite[Theorem 6.6]{chernikov2020hypergraph}):

\begin{cor}\label{cor: slice-wise packing lemma}
For any $\varepsilon,d$ there exists $N = N(\varepsilon, d)$ satisfying the following. 

	Let $(V_i, \mu_i)$ be probability spaces for $1 \leq i \leq k'$, and $E \subseteq V_1 \times \ldots \times V_{k'}$  has slice-wise $\VC_k$-dimension $\leq d$. Then there exist some $\bar{a}_1, \ldots, \bar{a}_N \in V_{[k']}$ so that: for all $\bar{z} \in V_{[k']\setminus [k+1]}$, $\mu_{[k+1]}(E_{z_{k+1}} \triangle D) \leq \varepsilon$ for some $D \subseteq V_{[k+1]}$ given by a Boolean combination of at most $N$ of $\leq k$-ary fibers of $E$ with all fixed coordinates coming from $\bar{a}_1, \ldots, \bar{a}_N$, and at most $N$ of $\leq (k-1)$-ary fibers of $E$ with their fixed coordinates possibly varying with $\bar{z}$.
\end{cor}
\begin{proof}
	Let $k',k, d$ and $\varepsilon$ be fixed. By Fact \ref{fac: slicewise reg lemma} (and Corollary \ref{cor: for all slices}) there is $N = N(d, \varepsilon)$ so that we have sets $S^t_{I} \subseteq \prod_{i \in I \cup ([k']\setminus [k+1]) } V_i$ for $t < N$ and $I \subseteq [k+1], |I| \leq k$ so that for all $\bar{z} = (z_i : i \in [k'] \setminus [k+1]) \in V_{ [k']\setminus [k+1]}$, 
	$$\mu_{[k+1]} \left(E_{\bar{z}} \triangle \left( \bigcup_{1 \leq t \leq N}  \bigcap_{I \subseteq [k+1], |I| \leq k}(\widetilde{S}^t_{I})_{\bar{z}}\right)  \right) \leq \varepsilon.$$
	
	Since boundedness of slice-wise $\VC_k$-dimension is preserved under permuting the coordinates, taking slices and under Boolean combinations \cite[Fact 3.3, Remark 3.5]{chernikov2020hypergraph},  by Fact \ref{fac: slicewise reg lemma}(2) there exists $d' = d'(d, N) = d'(d,\varepsilon)$ so that the slice-wise $\VC_k$-dimension of $S^t_{I}$ is at most $d'$ for all $t,I$. Applying the $\VC_k$-packing lemma (Fact \ref{fac: finite VCk-dim implies approx bounded}) to each $(k+1)$-ary relation $S^t_{I} \subseteq (\prod_{i \in I} V_i) \times (V_{[k'] \setminus [k+1]})$ and regrouping, it follows that there exist $N' = N'(N, \varepsilon) = N'(d,\varepsilon)$ and some $\bar{z}_1, \ldots, \bar{z}_{N'} \in V_{[k'] \setminus [k+1]}$ so that \emph{for every} $\bar{z} \in V_{[k'] \setminus [k+1]}$ we have $\mu_{[k+1]} \left(E_{\bar{z}} \triangle D \right) \leq \varepsilon$ for some  $D$ given by a Boolean combination of the $\leq k$-ary fibers $(\widetilde{S}^t_I)_{\bar{z}_i}$ and at most $N'$ of $\leq (k-1)$-ary fibers of the form $(\widetilde{S}^t_I)_{(\bar{z}_i,\bar{a})}$ or $(\widetilde{S}^t_I)_{(\bar{z},\bar{a})}$ for $\bar{a}$ varying with $\bar{z}$. As each  $S^t_{I}$ itself is a Boolean combination of at most $N$ fibers of $E$, the conclusion follows.
\end{proof}

For example, in the case $k=1$ and $k' = 3$, this specializes to the following:
\begin{cor}
	For every $d, \varepsilon$ there exists $N = N(d,\varepsilon)$ satisfying the following.  	If $(V_i, \mu_i)$ are probability spaces for $i \in [3]$ and  $E \subseteq V_1 \times V_2 \times V_3$ has slice-wise $\VC$-dimension $\leq d$ (i.e.~every slice of $E$ by fixing a single coordinate has VC-dimension $\leq d$), then there exist $z_1, \ldots, z_N \in V_3$ so that for every $z \in V_3$, $\mu_1 \otimes \mu_2 (E_z \triangle E_{z_i}) \leq \varepsilon$ for some $i \leq N$.
\end{cor}
\begin{proof}
	Applying Corollary \ref{cor: slice-wise packing lemma}, we find some $\bar{a}_1, \ldots, \bar{a}_{N} \in V_{[3]}$ so that for every $z \in V_3$, $E_z$ is within $\varepsilon$ of some Boolean combination of (cylinders over)  fibers of $E$ with $2$ out of $3$ coordinates fixed by some elements of $\bar{a}_1, \ldots, \bar{a}_N$ 	 (note that we no longer have any fibers of arity $< k = 1$ varying with $z$). Since there are at most $N' = N'(N)$ such combinations, we can pick one fiber $E_{z_i}$ for each such Boolean combination that occurs. Then every fiber $E_z$ is within $2\varepsilon$ of one of the fibers $E_{z_1}, \ldots, E_{z_{N'}}$.
\end{proof}

%
%
%
%
%
%
%
%
%
%
%
%
%
%
%
%
%

\bibliographystyle{alpha}
\bibliography{refs}
\end{document}